\documentclass[mlmain]{jmlr} 

\makeatletter
\renewcommand{\ps@jmlrtps}{
    \let\@mkboth\@gobbletwo
    \def\@oddhead{} 
    \let\@evenhead\@oddhead
    \def\@oddfoot{\@titlefoot}
    \let\@evenfoot\@oddfoot
}
\makeatother

\usepackage{tikz}
\usepackage{amsmath}
\usepackage{wrapfig}
\usetikzlibrary{arrows.meta, positioning, quotes}

\usepackage{booktabs}
\usepackage[load-configurations=version-1]{siunitx} 
\usetikzlibrary{calc}

\newcommand{\tabref}[1]{Table~\ref{#1}}

\theorembodyfont{\upshape}
\theoremheaderfont{\scshape}
\theorempostheader{:}
\theoremsep{\newline}

\jmlrvolume{}
\firstpageno{1}

\newcommand{\methodtitle}{Event2Vec}
\newcommand{\fig}[1]{Figure~\ref{#1}}

\title[\methodtitle: A Geometric Approach to Learning Composable Representations]{\methodtitle: A Geometric Approach to Learning Composable Representations of Event Sequences}
\author{
 \Name{Antonin Sulc}\\
 \Email{asulc@lbl.gov}\\
 \addr{Lawrence Berkeley National Laboratory\\
 Berkeley, CA, U.S.A.}
 }

\begin{document}

\maketitle

\begin{abstract}
The study of neural representations, both in biological and artificial systems, is increasingly revealing the importance of geometric and topological structures. Inspired by this, we introduce Event2Vec, a novel framework for learning representations of discrete event sequences. Our model leverages a simple, additive recurrent structure to learn composable, interpretable embeddings. We provide a theoretical analysis demonstrating that, under specific training objectives, our model's learned representations in a Euclidean space converge to an ideal additive structure. This ensures that the representation of a sequence is the vector sum of its constituent events, a property we term the linear additive hypothesis. To address the limitations of Euclidean geometry for hierarchical data, we also introduce a variant of our model in hyperbolic space, which is naturally suited to embedding tree-like structures with low distortion. We present experiments to validate our hypothesis. Quantitative evaluation on the Brown Corpus yields a Silhouette score of 0.0564, outperforming a Word2Vec baseline (0.0215), demonstrating the model's ability to capture structural dependencies without supervision.
\end{abstract}

\section{Introduction}

The convergence of neuroscience and machine learning has highlighted the critical role of geometry in shaping neural representations~\cite{bronstein2017geometric}. In the brain, neural activity often unfolds on low-dimensional manifolds, reflecting the underlying structure of tasks and environments~\cite{kriegeskorte2013representational}. Similarly, in artificial intelligence, principles like equivariance and compositionality are key to developing generalizable and interpretable models~\cite{cohen2016group}. This paper contributes to this effort by investigating how a simple geometric prior can lead to highly structured and interpretable representations of sequential data.

In this work\footnote{Code available at \url{https://github.com/sulcantonin/event2vec_public}}, we address the challenge of learning representations for sequences of discrete events by introducing \methodtitle, a model designed to learn geometrically structured and sequentially compositional representations. Our central hypothesis, which we term the~\textit{linear additive hypothesis}~\cite{mikolov2013exploiting}, is that the representation of an entire event history can be modeled as the vector sum of the embeddings of its constituent events. While based on addition, this structure enables a rich vector arithmetic, allowing for both the composition (via addition) and decomposition (via subtraction) of event trajectories. This high degree of interpretability allows for reasoning about entire trajectories; for instance, the displacement vector between first job and promotion (a subtraction) can represent the abstract concept of 'career progression'. Such a structure allows us to analyze a complex event trajectory by deconstructing it into its ordered, constituent steps, providing a clear path towards geometric mechanistic interpretability for sequential data.

We present two variants of our model: one operating in a standard Euclidean space and another in a hyperbolic space. Hyperbolic geometry is particularly well-suited for embedding hierarchical or tree-like structures with low distortion~\cite{nickel2017poincare}.
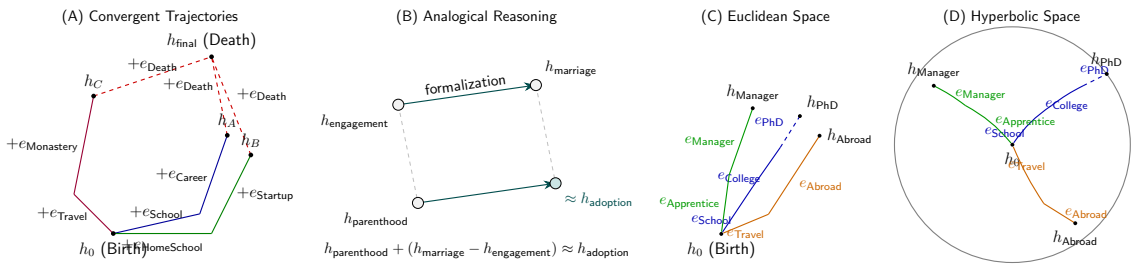
\begin{figure}[h!]
\sffamily
\resizebox{1.0\linewidth}{!}{%
\begin{tabular}{cccc}
(A) Convergent Trajectories & (B) Analogical Reasoning & (C) Euclidean Space & (D) Hyperbolic Space\\
\begin{tikzpicture}[
    font=\sffamily\large, 
    path1/.style={thick, draw=blue!60!black},
    path2/.style={thick, draw=green!50!black},
    path3/.style={thick, draw=purple!80!black},
    death_vec/.style={thick, draw=red!80!black, dashed},
    state_dot/.style={circle, fill=black, inner sep=1.2pt}
]
    \coordinate (h0) at (0,0);
    \coordinate (h_death) at (2.5, 4.5);

    \coordinate (p1_1) at (2.2, 0.5);
    \coordinate (p1_2) at (2.9, 2.5);
    \draw[path1] (h0) -- (p1_1) node[midway, above] {$+e_{\text{School}}$};
    \draw[path1] (p1_1) -- (p1_2) node[midway, left] {$+e_{\text{Career}}$};
    \draw[death_vec] (p1_2) -- (h_death) node[midway, above left] {$+e_{\text{Death}}$};
    \node[state_dot, label=above:{$h_A$}] at (p1_2) {};

    \coordinate (p2_1) at (2.5, 0);
    \coordinate (p2_2) at (3.5, 2.0);
    \draw[path2] (h0) -- (p2_1) node[midway, below] {$+e_{\text{HomeSchool}}$};
    \draw[path2] (p2_1) -- (p2_2) node[midway, right] {$+e_{\text{Startup}}$};
    \draw[death_vec] (p2_2) -- (h_death) node[midway, above right] {$+e_{\text{Death}}$};
    \node[state_dot, label=above:{$h_B$}] at (p2_2) {};

    \coordinate (p3_1) at (-1.0, 1.0);
    \coordinate (p3_2) at (-0.5, 3.5);
    \draw[path3] (h0) -- (p3_1) node[midway, left] {$+e_{\text{Travel}}$};
    \draw[path3] (p3_1) -- (p3_2) node[midway, left] {$+e_{\text{Monastery}}$};
    \draw[death_vec] (p3_2) -- (h_death) node[midway, above] {$+e_{\text{Death}}$};
    \node[state_dot, label=above:{$h_C$}] at (p3_2) {};

    \node[state_dot, label=below:$h_0$ (Birth)] at (h0) {};
    \node[state_dot, label=above:$h_{\text{final}}$ (Death)] at (h_death) {};
\end{tikzpicture} & 

\begin{tikzpicture}[
    font=\sffamily,
    state_dot/.style={circle, draw=black, fill=gray!10, minimum size=8pt, inner sep=0pt},
    analogy_vec/.style={thick, draw=teal!60!black, -{Stealth[length=3mm]}},
    connector/.style={thin, draw=gray!50, dashed}
]
    \coordinate (engagement) at (0,0);
    \coordinate (marriage) at (3.5, 0.5);
    \coordinate (parenthood) at (0.5, -2.5);
    
    \coordinate (adoption) at ($(parenthood) + (marriage) - (engagement)$);

    
    \draw[connector] (engagement) -- (parenthood);
    \draw[connector] (marriage) -- (adoption);

    \draw[analogy_vec] (engagement) -- (marriage) node[midway, above, sloped] {formalization};
    \draw[analogy_vec] (parenthood) -- (adoption);

    \node[state_dot, label=below left:$h_{\text{engagement}}$] at (engagement) {};
    \node[state_dot, label=above right:$h_{\text{marriage}}$] at (marriage) {};
    \node[state_dot, label=below left:$h_{\text{parenthood}}$] at (parenthood) {};
    
    \node[state_dot, fill=teal!20, label=below right:{\color{teal!60!black}$\approx h_{\text{adoption}}$}] at (adoption) {};

    \node[align=center, font=\normalsize] at (2, -3.7) 
        {$h_{\text{parenthood}} + (h_{\text{marriage}} - h_{\text{engagement}}) \approx h_{\text{adoption}}$};
\end{tikzpicture} &
\begin{tikzpicture}[
    font=\sffamily\large,
    state_dot/.style={circle, fill=black, inner sep=1pt},
    event_vec/.style={thick}
]
    \node[state_dot, label=below:{$h_0$ (Birth)}] (O) at (0,0) {};

    \coordinate (edu1) at (0.8, 1.2);
    \coordinate (edu2) at (1.5, 2.2);
    \coordinate (edu3) at (2.0, 3.0);
    \draw[event_vec, blue!70!black] (O) -- node[midway, below left, font=\normalsize] {$e_{\text{School}}$} (edu1);
    \draw[event_vec, blue!70!black] (edu1) -- node[midway, below left, font=\normalsize] {$e_{\text{College}}$} (edu2);
    \draw[event_vec, blue!70!black, dashed] (edu2) -- node[midway, above left, font=\normalsize] {$e_{\text{PhD}}$} (edu3);
    \node[state_dot,label=above right:{$h_{\text{PhD}}$}] at (edu3) {};

    \coordinate (car1) at (0.2, 1.5);
    \coordinate (car2) at (0.8, 3.2);
    \draw[event_vec, green!60!black] (O) -- node[midway, left, font=\normalsize] {$e_{\text{Apprentice}}$} (car1);
    \draw[event_vec, green!60!black] (car1) -- node[midway, left, font=\normalsize] {$e_{\text{Manager}}$} (car2);
    \node[state_dot,label=above:{$h_{\text{Manager}}$}] at (car2) {};

    \coordinate (trav1) at (1.2, 0.5);
    \coordinate (trav2) at (2.5, 2.5);
    \draw[event_vec, orange!80!black] (O) -- node[midway, below, font=\normalsize] {$e_{\text{Travel}}$} (trav1);
    \draw[event_vec, orange!80!black] (trav1) -- node[midway, below right, font=\normalsize] {$e_{\text{Abroad}}$} (trav2);
    \node[state_dot,label=right:{$h_{\text{Abroad}}$}] at (trav2) {};
    
\end{tikzpicture} &
\begin{tikzpicture}[
    font=\sffamily\large,
    state_dot/.style={circle, fill=black, inner sep=1pt},
    geodesic/.style={thick},
    disk/.style={draw=black!50, thick}
]
    \draw[disk] (0,0) circle (3cm);
    \node[state_dot, label=below:{$h_0$}] (O) at (0,0) {};

    \coordinate (edu1) at (1.0,1.0);
    \coordinate (edu2) at (1.8,1.5);
    \coordinate (edu3) at (2.4,1.8);
    \draw[geodesic, blue!70!black] (O) to[bend left=10] node[midway, below left] {\normalsize $e_{\text{School}}$} (edu1);
    \draw[geodesic, blue!70!black] (edu1) to[bend left=5] node[midway, below] {\normalsize $e_{\text{College}}$} (edu2);
    \draw[geodesic, blue!70!black, dashed] (edu2) to[bend left=0] node[midway, above] {\normalsize $e_{\text{PhD}}$} (edu3);
    \node[state_dot,label=above:{$h_{\text{PhD}}$}] at (edu3) {};

    \coordinate (car1) at (-1.2,1.0);
    \coordinate (car2) at (-2.0,1.5);
    \draw[geodesic, green!60!black] (O) to[bend right=10] node[midway, right] {\normalsize $e_{\text{Apprentice}}$} (car1);
    \draw[geodesic, green!60!black] (car1) to[bend right=5] node[midway, right] {\normalsize $e_{\text{Manager}}$} (car2);
    \node[state_dot,label=above:{$h_{\text{Manager}}$}] at (car2) {};

    \coordinate (trav1) at (0.8,-1.5);
    \coordinate (trav2) at (1.6,-2.0);
    \draw[geodesic, orange!80!black] (O) to[bend right=5] node[midway, above] {\normalsize $e_{\text{Travel}}$} (trav1);
    \draw[geodesic, orange!80!black] (trav1) to[bend right=0] node[midway, right] {\normalsize $e_{\text{Abroad}}$} (trav2);
    \node[state_dot,label=below:{$h_{\text{Abroad}}$}] at (trav2) {};
\end{tikzpicture}
\end{tabular}}
\caption{
        \textbf{Geometric Properties of \methodtitle{}~Embeddings.} The model learns to represent event sequences as trajectories in a vector space. 
        \textbf{(A)} Probable event sequences form trajectories where consecutive event vectors are directionally aligned. For example, two distinct life paths (`homeschool' $\rightarrow$ `career' and `startup' $\rightarrow$ `career') converge towards a similar state (`death') by adding event vectors that follow a logical progression. 
        \textbf{(B)} The additive structure enables analogical reasoning through vector arithmetic. The "formalization" vector learned from the `engagement` to `marriage` transition can be applied to `parenthood` to correctly identify the parallel concept of `adoption`. 
        \textbf{(C)} Euclidean Space: Unlike context-based models which cluster by semantic similarity, Event2Vec arranges events by sequential proximity. Distant points on the plot represent distant points in time, preserving the chronological trajectory.
        \textbf{(D)} Hyperbolic Space: With its exponential volume, it is better suited for capturing hierarchical branching, preventing the 'crowding' of distinct paths.}
\label{fig:teaser}
\end{figure}

\section{Related Work}

The concept of vector space embeddings for words, such as Word2Vec~\cite{mikolov2013distributed} and GloVe~\cite{pennington2014glove}, demonstrated the power of capturing semantic relationships through vector arithmetic. These models are trained on local co-occurrence statistics within a fixed context window. However, their reliance on local co-occurrence statistics within a fixed context window makes them unsuitable for modeling the long-range, ordered dependencies found in event sequences. In contrast, our model's recurrent architecture and novel reconstruction loss are designed to explicitly capture this directed, temporal structure, enforcing sequential compositionality over the entire event history.

Our work is related to methods for learning node representations in graphs, such as DeepWalk~\cite{perozzi2014deepwalk} and Node2Vec~\cite{grover2016node2vec}. These methods generate sequences of nodes through random walks on the graph and then use them to learn embeddings. 
However, these approaches treat the generated walks as unordered contexts. They are designed to capture neighborhood structure rather than the explicit, directed, and temporal order of events in a sequence, which is the primary focus of our model.

Recurrent Neural Networks (RNNs) and their variants, like Long Short-Term Memory (LSTM)~\cite{hochreiter1997long} and Gated Recurrent Units (GRUs)~\cite{cho2014learning}, are the standard for modeling sequential data. While powerful, their complex, non-linear dynamics, involving gating mechanisms and matrix multiplications, can lead to representations that are difficult to interpret. Other powerful sequence models like Neural Temporal Point Processes~\cite{shchur2021neural, mei2017neural,du2016recurrent} also learn embeddings from event histories, their complex, non-linear representations are optimized for temporal prediction and are not designed to support the kind of simple, interpretable vector arithmetic that our model explicitly enforces. 


Our hyperbolic prototype is directly inspired by recent advances showing that hyperbolic spaces are particularly effective for embedding hierarchical data. Their exponential volume growth naturally accommodates tree-like structures with low distortion, a property first leveraged for representation learning in Poincaré embeddings~\cite{nickel2017poincare} and later extended to models that operate within curved geometry~\cite{ganea2018hyperbolic}.

Our approach introduces a novel contribution by bridging the gap between purely sequential models and context-based embedding methods. 
While RNNs capture order and methods like Word2Vec capture context, our model is, to our knowledge, the first to explicitly unify both within a simple, geometrically interpretable framework. 
The core novelty lies in our training objective, which learns representations that are not just contextually similar but are also compositionally sequential. 
Specifically, our loss function encourages the model to place likely subsequent events linearly along a trajectory in the embedding space, while positioning semantically unrelated or unlikely events in orthogonal directions. This creates a much richer geometric structure where the vector displacement between events is directly tied to their sequential probability, a property not explicitly enforced by prior methods.

\section{The \methodtitle{}~Model}

We model a sequence of events $S = (s_1, s_2, ..., s_T)$, where each event $s_t$ is drawn from a vocabulary of event types. Our goal is to learn an embedding vector $e_i \in \mathbb{R}^d$ for each event type $i$, and a hidden state vector $h_t \in \mathbb{R}^d$ that represents the history of events up to time $t$.

\subsection{Model Architecture}

The core of the \methodtitle{} model is its additive state update mechanism, where the event embedding $e_{s_t}$ encodes a transition that is applied to the previous state $h_{t-1}$. We propose two geometric variants: a Euclidean model using standard vector addition, and a Hyperbolic model using Möbius addition, which is better suited for hierarchical data.

\paragraph{Euclidean Model}
In the Euclidean variant, the update rule is a pure vector addition. However, simple additivity can be numerically unstable, as the magnitude of the hidden state vector can grow without bound over long sequences. To ensure stable training, we employ a regularization technique. After the additive update, we clip the L2 norm of the resulting hidden state vector to a maximum value. This preserves the direction of the vector, and thus the additive semantics, while preventing its magnitude from causing training instabilities.
\begin{equation}
    h_t = h_{t-1} + e_{s_t} \quad (\text{with norm clipping})
    \label{eq:euclidean_update}
\end{equation}
The initial hidden state $h_0$ is a zero vector. The model predicts the next event $s_{t+1}$ via a linear decoder followed by a softmax function:
\begin{equation}
    P(s_{t+1} | h_t) = \text{softmax}(W_{dec} h_t + b_{dec})
\end{equation}

\paragraph{Hyperbolic Model}
For the hyperbolic variant, we operate in the Poincaré ball model from~\cite{ganea2018hyperbolic}, a conformal model of hyperbolic space that represents points within a unit ball. The additive update is replaced by Möbius addition $\oplus_c$, the natural generalization of addition to hyperbolic space with curvature $c < 0$. It is defined as:
\begin{equation}
    x \oplus_c y = \frac{(1 + 2c\langle x, y \rangle + c\|y\|^2)x + (1 - c\|x\|^2)y}{1 + 2c\langle x, y \rangle + c^2\|x\|^2\|y\|^2}
\end{equation}
This operation ensures that the addition of two vectors within the Poincaré ball results in another vector that also lies within the ball.\footnote{We note that while Euclidean addition is commutative, Möbius addition defines a Gyrovector space which is gyro-commutative and non-associative. However, our step-wise reconstruction loss effectively enforces the local inverse property required for compositionality.} The update rule is:
\begin{equation}
    h_t = h_{t-1} \oplus_c e_{s_t}
\end{equation}
To make a prediction, we must interface the hyperbolic representation with a standard Euclidean classifier. We do this by mapping the hyperbolic state $h_t$ to the tangent space at the origin (a Euclidean space) via the logarithmic map, $\log_{0,c}(\cdot)$, before applying the linear transformation. This is a standard and principled approach for decoding from hyperbolic representations.

\subsection{Training Objective}

The model is trained to minimize a composite loss function $\mathcal{L}_{total}$, which consists of three components designed to enforce our desired geometric properties:

\begin{enumerate}
    \item \textbf{Prediction Loss ($\mathcal{L}_{pred}$):} A standard cross-entropy loss for predicting the next event in the sequence. This ensures the model learns to capture the sequential dependencies in the data.
    \begin{equation}
        \mathcal{L}_{pred} = -\sum_{t=1}^{T-1} \log P(s_{t+1} | h_t)
        \label{eq:Lpred}
    \end{equation}

    \item \textbf{Reconstruction Loss ($\mathcal{L}_{recon}$):} This loss is central to our linear additive hypothesis. It ensures that the event embedding can be "subtracted" from the hidden state to perfectly recover the previous state. This explicitly enforces the algebraic group structure of addition and its inverse.
    \begin{equation}
        \mathcal{L}_{recon} = \sum_{t=1}^{T} \| (h_t - e_{s_t}) - h_{t-1} \|^2_2
        \label{eq:Lrecon}
    \end{equation}
    For the hyperbolic case, this becomes $\mathcal{L}_{recon} = \sum_t d_c(h_t \oplus_c (-e_{s_t}), h_{t-1})^2$, where $d_c$ is the Poincaré distance.

    \item \textbf{Consistency Loss ($\mathcal{L}_{consist}$):} To ensure robustness, we use dropout as a stochastic perturbation~\cite{gao2021simcse}, penalizing differences between hidden states from two forward passes of the same input: $\mathcal{L}_{consist} = \sum_{t=1}^{T} \| h_t^{(1)} - h_t^{(2)} \|^2_2$, where $h_t^{(1)}$ and $h_t^{(2)}$ correspond to different dropout masks. 
\end{enumerate}

The total loss is a weighted sum of these components:
\begin{equation}
    \mathcal{L}_{total} = \mathcal{L}_{pred} + \lambda_{recon}\mathcal{L}_{recon} + \lambda_{consist}\mathcal{L}_{consist}
\end{equation}

\section{Theoretical Justification for Additivity}

Here, we justify that for the Euclidean \methodtitle{} model, minimizing our composite loss function forces the recurrent update to converge to an ideal additive form. This provides a theoretical foundation for our claim that the model learns composable representations.

\begin{theorem}[Justification for Ideal Additivity]
Let the hidden state update be defined as $h_t = f(h_{t-1}, e_{s_t})$, where $f$ is a function parameterized by the neural network. Minimizing the reconstruction loss $\mathcal{L}_{recon} = \sum_{t} \| f(h_t, -e_{s_t}) - h_{t-1} \|^2_2$ with respect to the parameters of $f$ drives $f$ to approximate the linear additive function $h_t = h_{t-1} + e_{s_t}$.
\end{theorem}

\begin{proof}
The reconstruction loss $\mathcal{L}_{recon}$ reaches its minimum of zero only when its argument is zero for all steps. This imposes the constraint $f(h_t, -e_{s_t}) = h_{t-1}$. By substituting the definition of the update step $h_t = f(h_{t-1}, e_{s_t})$ into this constraint, we obtain the nested identity:
$$
f(f(h_{t-1}, e_{s_t}), -e_{s_t}) = h_{t-1}
$$
This equation dictates that the update operation must be strictly reversible: the effect of applying an event embedding ($e_{s_t}$) is perfectly undone by applying its inverse ($-e_{s_t}$). The simplest function satisfying this property in a vector space is linear addition, $f(h, e) = h + e$.

To address the possibility of scalar multiplicity, consider a candidate linear function $h_t = h_{t-1} + c \cdot e_{s_t}$, where $c$ is a scalar scaling factor. Substituting this into the nested identity yields:
$$
(h_{t-1} + c \cdot e_{s_t}) + c \cdot (-e_{s_t}) = h_{t-1} \implies h_{t-1} = h_{t-1}
$$
This equality holds true for any non-zero scalar $c$. However, the event embedding $e_{s_t}$ is itself a learned parameter. Due to the scale invariance of the linear term, any solution with $c \neq 1$ can be equivalently re-parameterized as a solution with $c=1$ by absorbing the scalar into the embedding magnitude (i.e., defining a new effective embedding $e'_{s_t} = c \cdot e_{s_t}$). Therefore, without loss of generality, we adopt the canonical form $c=1$.
\end{proof}

It is important to note that while the reconstruction loss enforces this reversible structure, it does not by itself guarantee a meaningful solution. For instance, a trivial solution where all event embeddings are zero ($e_t = \mathbf{0}$) would perfectly satisfy the reversibility constraint. This outcome is prevented by the prediction loss, which compels the model to learn informative, non-zero embeddings to make accurate forecasts about the sequence. While the reconstruction loss enforces the algebraic structure of additivity, the prediction loss imposes the semantic grounding for the embeddings. The two work in tandem.

\begin{theorem}[Semantic Grounding via Prediction Loss]
The prediction loss $\mathcal{L}_{pred}$ encourages the inner product of a hidden state $h_t$ and a subsequent event embedding $e_{s_{t+1}}$ to be proportional to their pointwise mutual information (PMI).
\end{theorem}
\begin{proof}[Sketch]
The $\mathcal{L}_{pred}$ provides the semantic grounding for the embeddings by linking them to the statistical properties of the event sequences. The objective is analogous to the Skip-Gram with Negative-Sampling (SGNS) framework~\cite{levy2014neural}, which has been shown to implicitly factorize a word-context matrix where each cell is the shifted Pointwise Mutual Information (PMI). The optimal solution for the dot product of a word vector and its context vector under the SGNS objective is $\vec{w}^T\vec{c} = \text{PMI}(w,c) - \log k$.

By analogy, in our model, the next event $s_{t+1}$ acts as the "word" and the history vector $h_t$ acts as the "context." Therefore, at the optimum, the model learns embeddings such that $h_t^T e_{s_{t+1}} \approx \text{PMI}(s_{t+1}, h_t) - \log k$. A crucial feature of our model is the linear additive hypothesis, where the history vector is the sum of its constituent event embeddings $h_t = \sum_{i=1}^{t}e_{s_{i}}$ (additive hypothesis). Substituting this into the relationship yields:
$$h_t^T e_{s_{t+1}} = \underbrace{
\left(\sum_{i=1}^{t} e_{s_i}\right)^T e_{s_{t+1}}
}_\text{Additive hypothesis} = 
\underbrace{
\sum_{i=1}^{t} e_{s_i}^T e_{s_{t+1}}
}_\text{Distributive property} \approx
\underbrace{
\text{PMI}(s_{t+1}, \{s_1, \dots, s_t\})
}_{\text{by analogy to the SGNS objective}}$$
This result demonstrates that the prediction loss $\mathcal{L}_{pred}$ forces the embeddings to be geometrically arranged such that events statistically likely to follow a sequence are directionally aligned with that sequence's vector. This grounds the learned geometry in the statistical reality of the data.
\end{proof}
The prediction and reconstruction losses work in synergy to create a space that is both semantically meaningful and structurally coherent. The prediction loss first provides the semantic grounding, arranging the space by ensuring that embeddings for statistically likely event sequences (e.g., `birth` then `infancy`) are geometrically close. 
However, this proximity alone does not enforce a specific algebraic rule. 
The reconstruction loss then acts as a structural constraint, taking this semantically organized space and "cementing" it with the rules of vector arithmetic. By forcing the event update to be perfectly reversible, $\mathcal{L}_{recon}$ ensures that the representation of a sequence is the precise vector sum of its parts (e.g., $h(\text{`birth`},\text{`infancy`}=e_{birth}+e_{infancy}$), making the space interpretably compositional.

\begin{theorem}[Hyperbolic Equivalence]
In the Poincaré ball, minimizing the hyperbolic reconstruction loss $\mathcal{L}_{recon} = \sum_t d_c(h_t \oplus_c (-e_{s_t}), h_{t-1})^2$ forces the update rule to be Möbius addition.
\end{theorem}
\begin{proof}
The proof follows the same logic as the Euclidean case, but with hyperbolic operations. The distance $d_c(x, y)$ is minimized when $x=y$. Thus, the loss is minimized when $h_t \oplus_c (-e_{s_t}) = h_{t-1}$. In hyperbolic geometry, the inverse of Möbius addition with a vector $e$ is Möbius addition with its negative, $-e$. Let the update rule be $h_t = f_c(h_{t-1}, e_{s_t})$. The loss minimum requires $f_c(f_c(h_{t-1}, e_{s_t}), -e_{s_t}) = h_{t-1}$. The function that satisfies this property in the Poincaré ball is Möbius addition, $f_c(h, e) = h \oplus_c e$. Therefore, minimizing the hyperbolic reconstruction loss forces the model to learn the correct additive compositional rule for that geometry.
\end{proof}

\section{Experiments}

To validate our model and its theoretical underpinnings, we propose the following experiments designed to rigorously test our claims and demonstrate the utility of our approach. In all experiments, the loss weights $\lambda$ were set to 1.

\subsection{Life Path Example}

To demonstrate the practical utility and interpretability of \methodtitle{}, we apply it to a synthetic dataset modeling a human life path. This dataset consists of sequences of significant life events, from birth and education to career milestones and retirement. 
The goal of this experiment is to qualitatively assess whether our model can learn a geometrically coherent representation of a chronological trajectory, a task that requires capturing not just the semantic meaning of events but also their inherent temporal order. Appendix explains the entire scenario. 

\begin{table}[h!]
\centering
\resizebox{\linewidth}{!}{%
\begin{tabular}{|c|ccccccc|c|}
\hline
\textbf{ID} & \multicolumn{5}{c}{\textbf{Analogy Query} ($A - B + C$)} &  & \textbf{Result} ($D$) & Cosine sim.\\
\hline
\hline
1A & `elementary\_school` & - &  `birth` & + & `first\_job` & = & `late\_childhood` & 0.4504\\
1B & `death` & - & `retirement` & + & `graduation`  & = & `internship` & 0.4090\\
2A & `promotion`& - & `career\_start` & + & `first\_job`  & = & `military\_service` & 0.5048\\
2B & `business\_success` & - & 'entrepreneurship' & + & 'investment'  & = & `inheritance`  & 0.5976\\
3A & `marriage` & - & 'engagement` & + & `parenthood`  & = & `adoption` & 0.4653 \\
3B & `divorce` & - & `relationship\_challenge` & + & `financial\_hardship`  & = & `major\_purchase` & 0.3869\\
\hline
\end{tabular}}
\caption{Analogical reasoning examples demonstrating that vector arithmetic ($A - B + C$) on \methodtitle{}~embeddings captures coherent semantic relationships between life events.}
\label{tab:analogy}
\end{table}


\paragraph{Geometric Analysis of Embedding} 
The model's learned geometry, shown in \fig{fig:embedding}, reveals distinct life trajectories. A primary path flows from an initialization phase of education (`birth`, `study\_abroad`) into a dense central hub representing a conventional career and family life (`marriage`, `leadership\_role`), with logical offshoots for events like `divorce` and `grandparenthood`. The visualization also captures alternative paths, such as a specialized academic sequence and individualistic pursuits (`travel`, `volunteer\_work`). Critically, the event `death` is not a single endpoint but is located centrally, reflecting its status as a shared outcome. However, this static visualization does not convey the full power of \methodtitle{}; its compositional additivity, which enables arithmetic reasoning on trajectories, is more robustly validated by the quantitative results in \textbf{\fig{fig:additivity} (Appendix)} and the Brown Corpus experiment in the following section.

\begin{figure}[h!]
    \centering
    \resizebox{1.0\linewidth}{!}{
    \begin{tabular}{ccc}
        {\small Euclidean \methodtitle{}} & {\small Hyperbolic \methodtitle{}} & {\small Word2Vec}\\
         \includegraphics[width=0.3\linewidth]{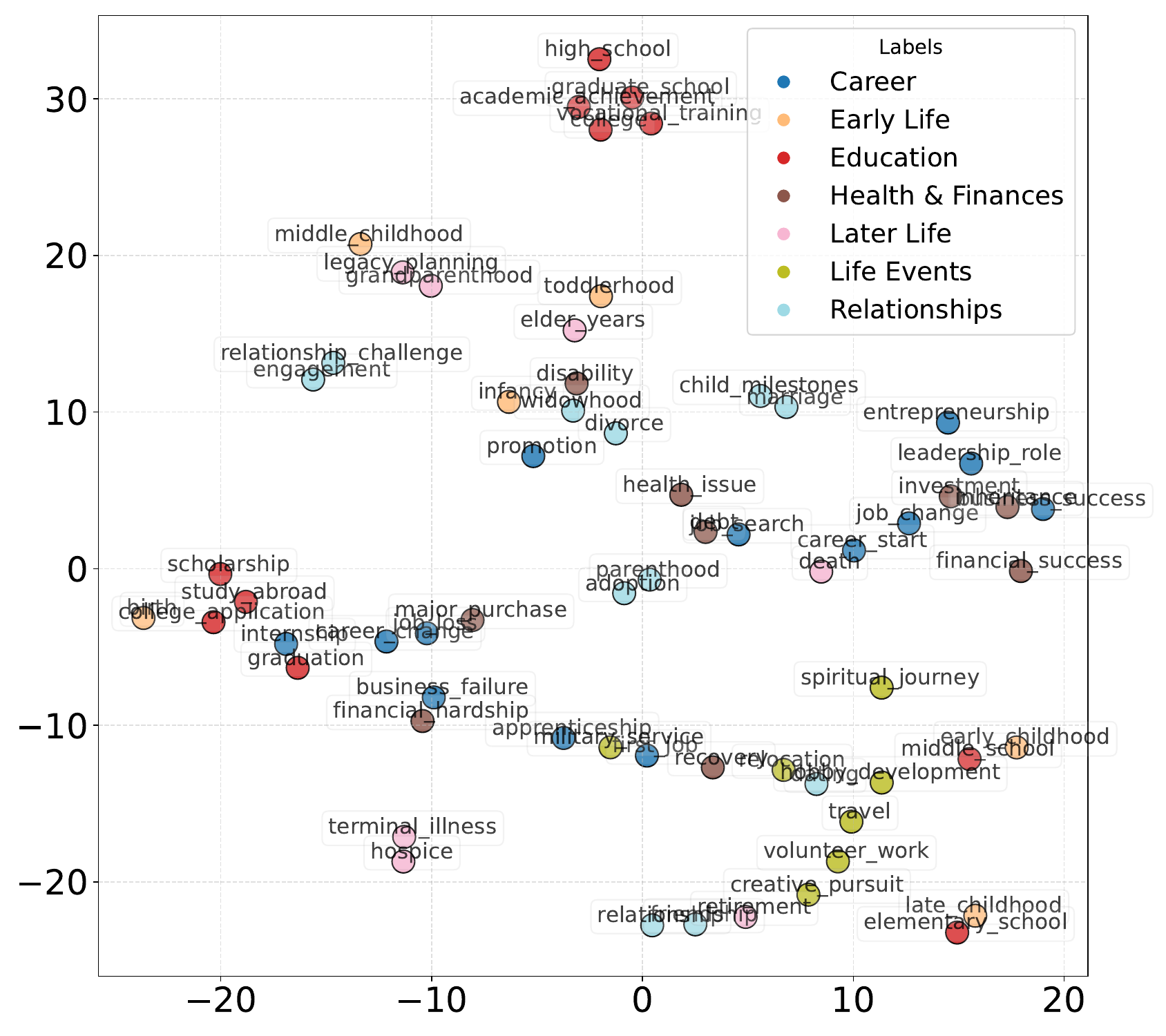} &
         \includegraphics[width=0.3\linewidth]{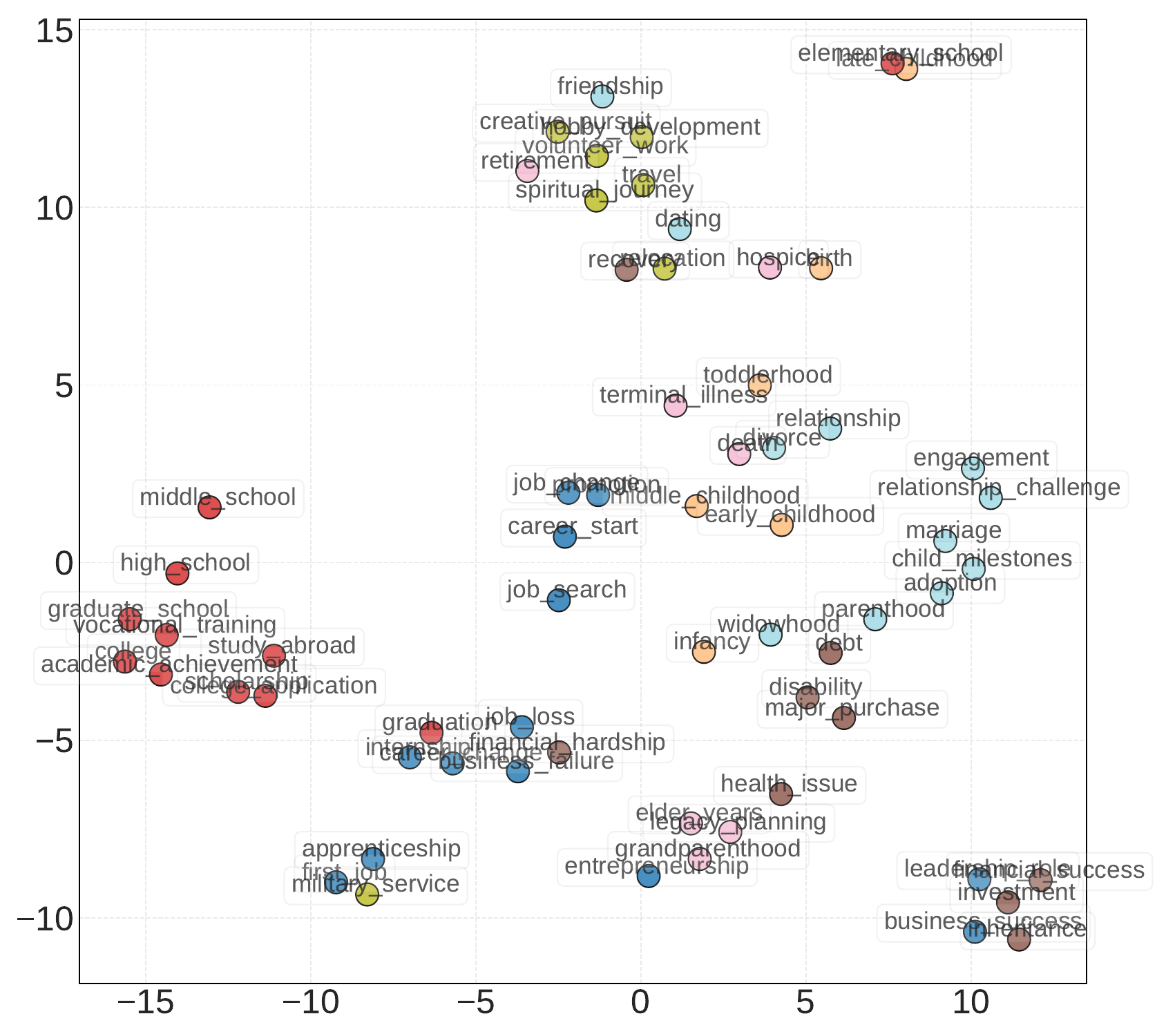} &
         \includegraphics[width=0.3\linewidth]{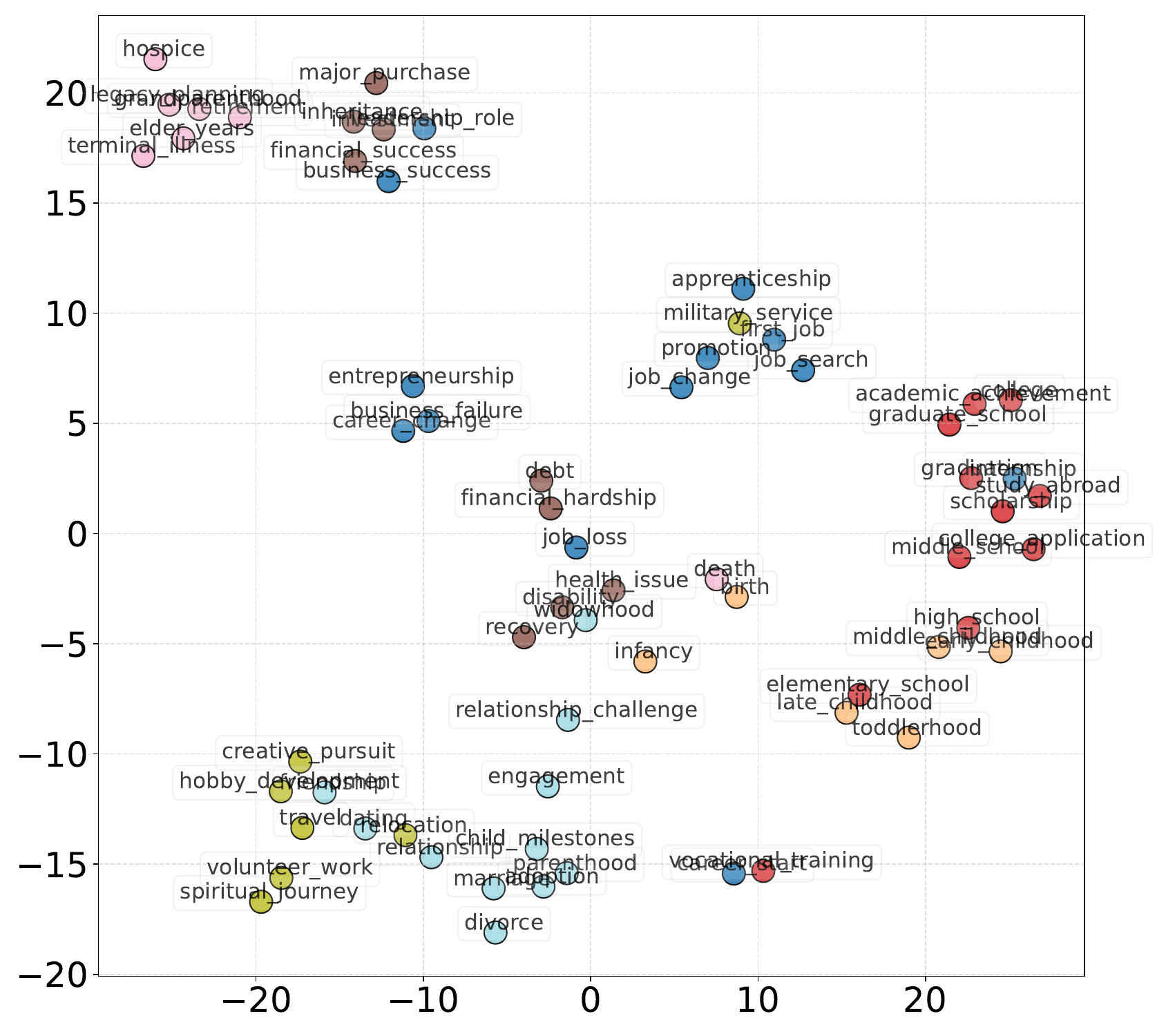}
         
    \end{tabular}}
    \caption{A t-SNE comparison of life path embeddings. \textbf{(Left)} The Euclidean \methodtitle{} model learns a clear chronological trajectory. 
    \textbf{(Center)} The Hyperbolic \methodtitle{} model captures a more powerful hierarchical structure, with life paths branching radially from the `birth` event. 
    \textbf{(Right)} The Word2Vec baseline learns thematic clusters, grouping events by semantic context (e.g., `birth` and `death`) while failing to capture sequential or hierarchical relationships.}
    \label{fig:embedding}
\end{figure}

\paragraph{Comparison of Euclidean vs. Hyperbolic Space:}
We trained a hyperbolic variant of \methodtitle{} to better model the inherently hierarchical nature of life paths, leveraging a geometry designed for tree-like data. The resulting embedding space, visualized using Poincaré distances, suggests a branching structure where life events appear to separate into domains like Education, Relationships, and Career. We hypothesize that this structure may be more effective at representing the divergent nature of life choices than a standard Euclidean trajectory, while still preserving the temporal structure. 

\subsection{Brown Corpus}

\label{exp:brown}
To validate our model's ability to capture linguistic structure, we conducted an unsupervised experiment on the Brown Corpus~\cite{bird2009natural}. We trained \methodtitle{} and a Word2Vec baseline on raw sentences without any grammatical labels. Subsequently, we evaluated the models by composing vector representations for specific Part-of-Speech (POS) tag sequences (e.g., `AT-JJ-NN` for article-adjective-noun) by summing the learned embeddings of corresponding words. This directly tests if the model's additive property learns meaningful syntactic composition.

The results, visualized in \fig{fig:embedding_brown}, show that \methodtitle{} successfully organizes these composite vectors into distinct clusters corresponding to their grammatical structure. Notably, it groups similar patterns like `AT-JJ-NN` and `IN-AT-NN`, a separation that is far less distinct in the Word2Vec baseline. This visual finding is quantitatively supported by a \textbf{silhouette score of 0.0564}, more than double the \textbf{0.0215} achieved by Word2Vec. This demonstrates that \methodtitle{} can discover and represent grammatical regularities from raw text without any explicit supervision.

\begin{figure}[h!]
\centering
\resizebox{1.0\linewidth}{!}{
\begin{tabular}{cc}
\includegraphics[width=0.5\linewidth]{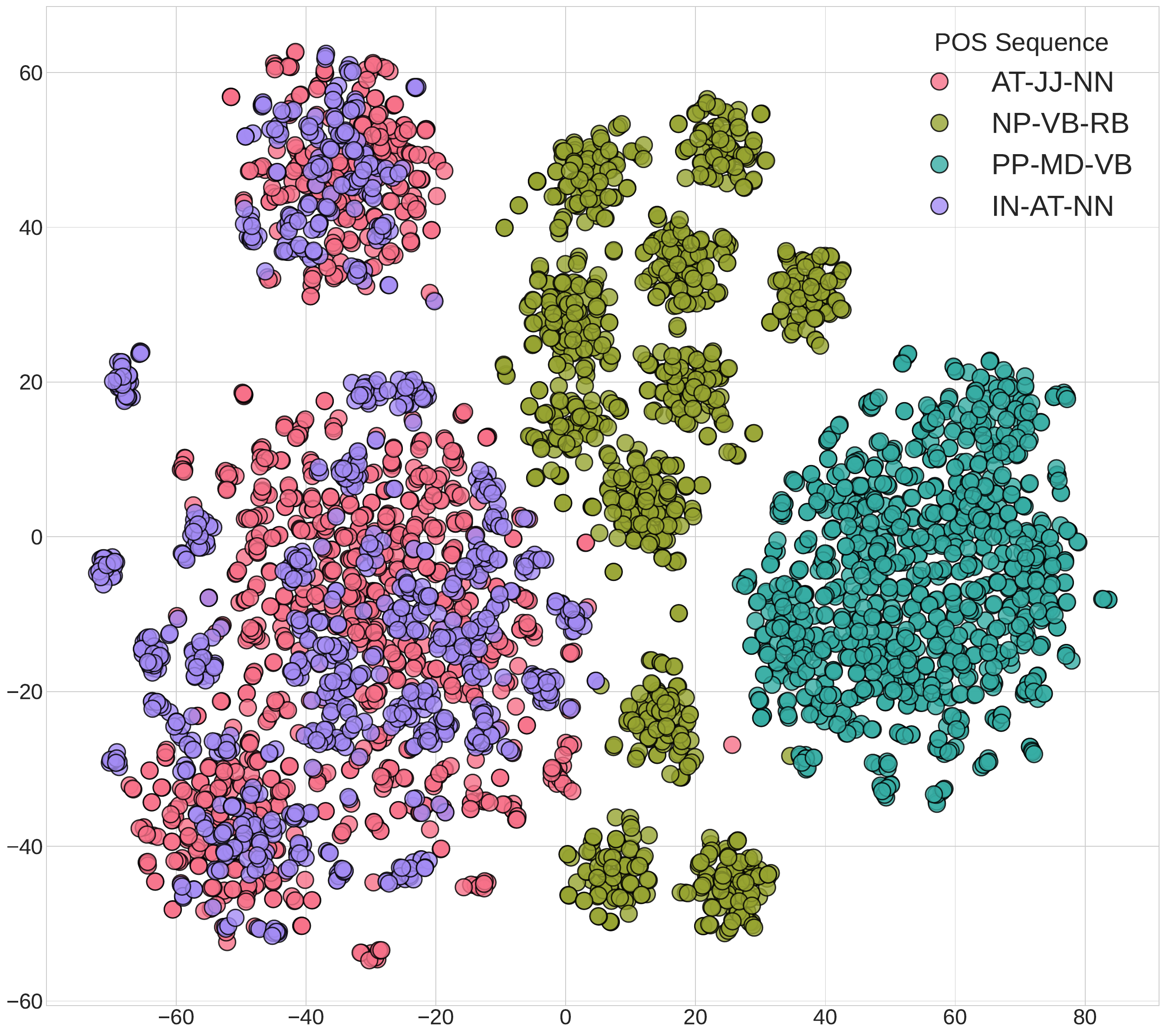} & \includegraphics[width=0.5\linewidth]{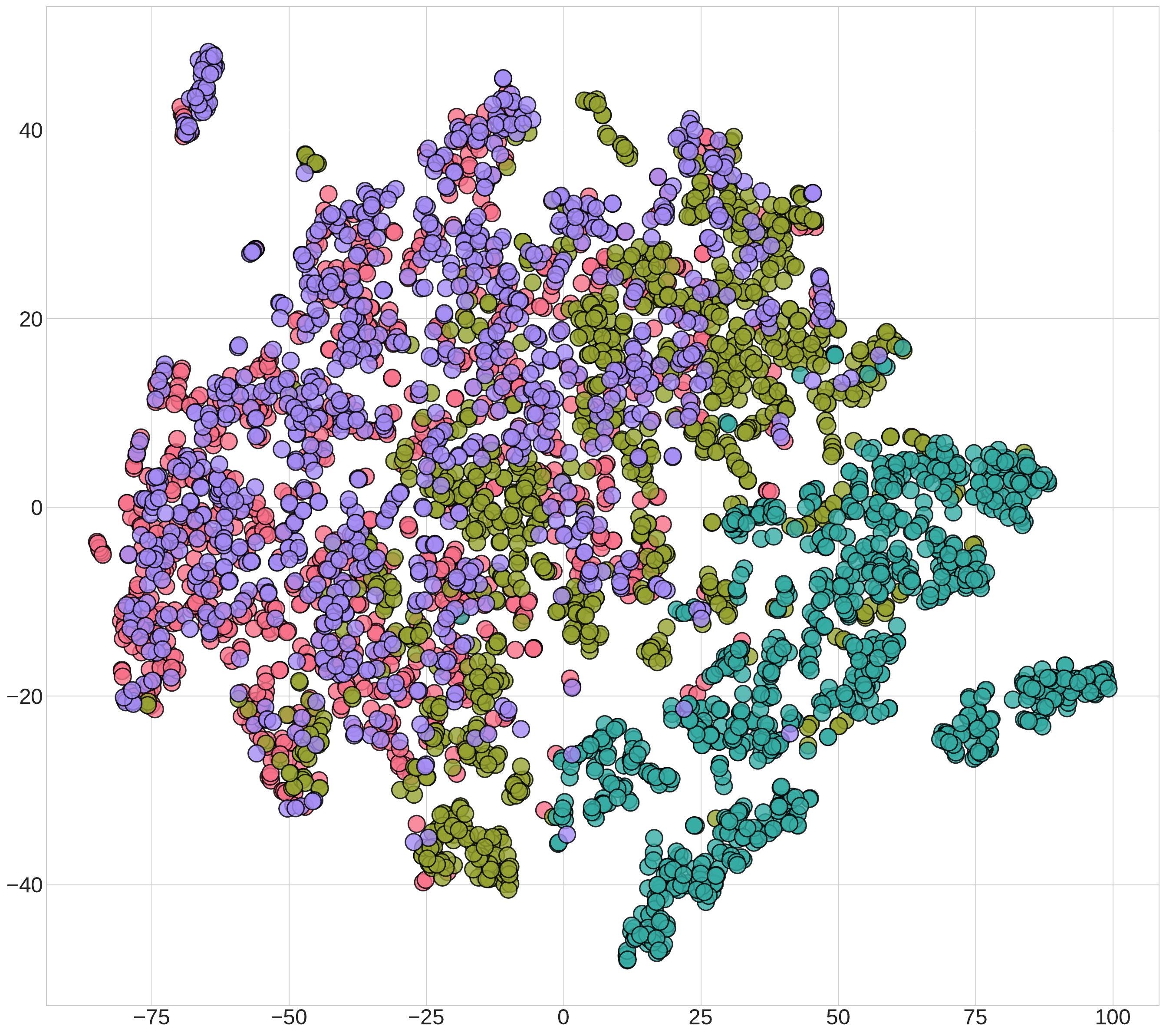}
\end{tabular}}
\caption{t-SNE visualization of embedded POS tag sequences from the Brown Corpus. \methodtitle{} (left) produces well-defined clusters corresponding to grammatical structures, correctly grouping similar sequences like `AT-JJ-NN` and `IN-AT-NN`. The Word2Vec baseline (right) shows significantly weaker separation and overlapping clusters.}
\label{fig:embedding_brown}
\end{figure}

\section{Limitations}

The model's deliberate simplicity and reliance on a additive structure introduce two key limitations. 
First, the pure additive mechanism is numerically unstable on long sequences, as the hidden state vector can grow without bound (see \fig{fig:additivity}). This necessitates regularization techniques that, in turn, compromise the model's perfect additivity. 
Second, the additive prior is a conceptual constraint, as it cannot capture complex, non-linear interactions. For instance, while a real-world `market\_crash` acts as a multiplicative operator that scales an `investment` proportional to its magnitude, our model is forced to represent it as a fixed displacement vector independent of the accumulated history. This prevents the model from capturing such conditional or magnitude-dependent dynamics.

\section{Conclusion}

We introduced \methodtitle{}, a model that learns composable representations of event sequences by enforcing geometric simplicity. Our theoretical analysis justifies that our training objective guides the model to learn a representation space with an ideal additive structure. This structure offers a clear path toward mechanistic interpretability, a claim supported by our quantitative results on the Brown Corpus where \methodtitle{} successfully clustered grammatical structures without supervision.

We also investigated a hyperbolic variant for embedding hierarchical, tree-like data. We note that unlike the commutative Euclidean model, this variant relies on Möbius addition, forming a non-associative, non-commutative Gyrovector space.

Finally, we acknowledge the trade-offs inherent in our design. While central to interpretability, our additive simplicity introduces limitations in numerical stability on long sequences. This work establishes a framework for geometrically-grounded sequence representation, paving the way for hybrid models that balance this interpretability with the capacity for more complex dynamics.

\section*{Acknowledgement}
This work was supported by the Director of the Office of Science of the U.S. Department of Energy under Contract No. DE-AC02-05CH11231

\clearpage
\bibliography{references}

\clearpage
\appendix
\subsection*{Experiments}
\subsubsection*{Life Scenario}
\label{sec:app:life}
The following figures show states and transitions of the life scenario as described by our model. We begin with  \fig{fig:blockoverview} which provides an overview of the individual blocks shown in further detail below. 

We see the education in \fig{fig:education_path}, which maps out the paths through the educational system. 
From there, \fig{fig:career} details the professional life. 
Running parallel to this, \fig{fig:relationship} illustrates the transition to states of e.g. finding partners and building a family. 
Personal choices shown in \fig{fig:lifeevents} can alter this path, while the critical interplay of wellness and wealth is detailed in \fig{fig:healthfinances}. 
The scenario transition concludes in \fig{fig:laterlife}, which shows the transition from retirement into the elder years. In each of these detailed figures, the arrows show the chance of moving from one event to the next.

\subsection*{Dataset Generation}

The synthetic Life Path dataset was generated using a stochastic process designed to simulate realistic life trajectories. The generation follows a guided random walk on a state transition graph, where nodes represent life events.

The process for generating a single sequence is as follows:
\begin{itemize}
\item \textbf{Initialization}: Each sequence begins at the predefined initial state, `birth`.
\item \textbf{Probabilistic Transition}: At each step, the next event is chosen based on the current event. For a given event, there is a predefined dictionary of possible subsequent events, each with an assigned probability weight. The next event is selected by sampling from this distribution.
\item \textbf{Stochastic Exploration}: To ensure diversity in the generated sequences and model less common life paths, a 10\% chance of a random transition is introduced at every step. In this case, instead of following the weighted probabilities, the model selects the next event uniformly at random from the entire vocabulary of possible life events.
\item \textbf{Termination}: The sequence generation continues until the terminal state, `death`, is reached or a maximum sequence length of 16 events is exceeded. If the maximum length is reached, the 'Death' event is appended to formally conclude the trajectory.
\end{itemize}
This procedure was used to generate 10,000 unique life path sequences, which formed the training set for the experiments described in Section 5.1.

\begin{figure}[h!]
    \centering
    \includegraphics[width=1.0\linewidth]{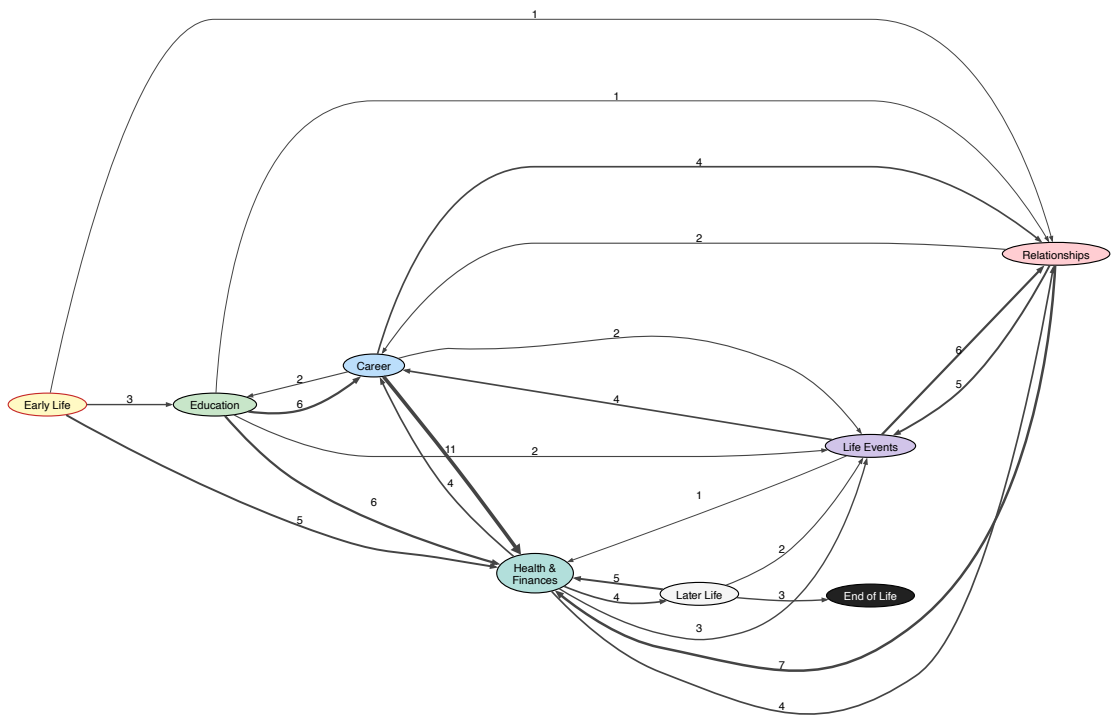}
    \caption{A high-level overview of the life model, where each node represents a major life stage. The directed edges indicate the flow between stages, with the label and thickness of each edge representing the total number of distinct transition paths. This illustrates the central roles of the \textbf{Career} and \textbf{Health \& Finances} stages as highly interconnected hubs in a person's life journey.}
    \label{fig:blockoverview}
\end{figure}

\begin{figure}[h!]
    \centering
    \includegraphics[width=1.0\linewidth]{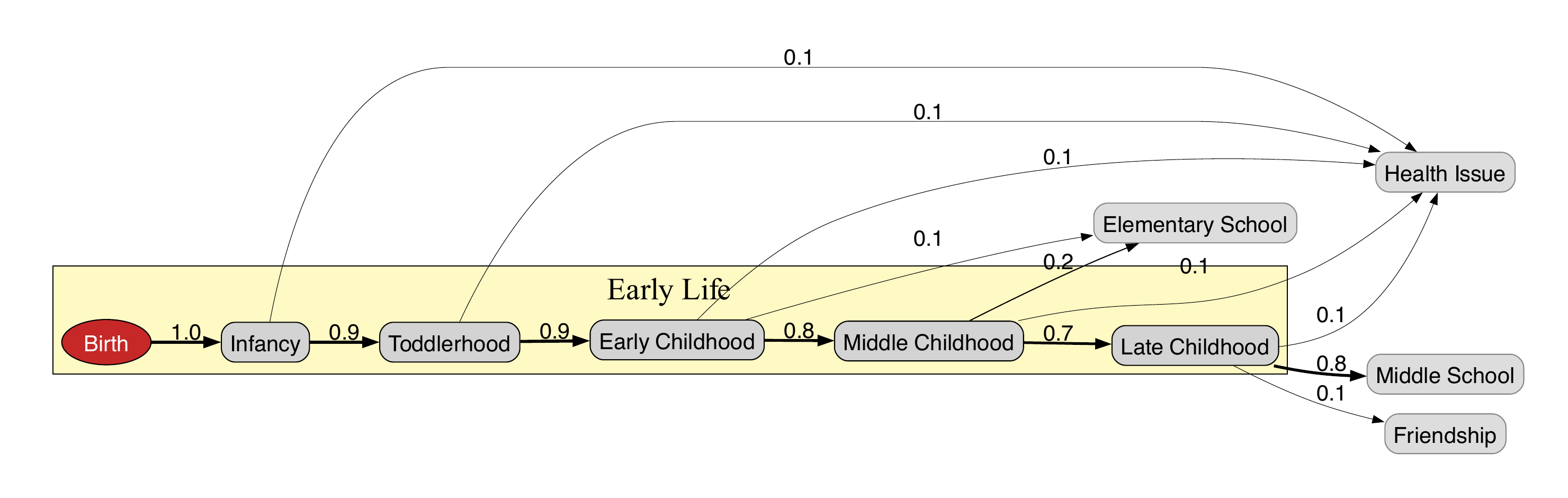}
    \caption{The Early Life stage, showing the foundational years from the `birth` event through childhood. The diagram shows a strong, linear progression through key developmental milestones, from `infancy` to `Late Childhood`, with high transition probabilities indicating a canonical path. The stage concludes with transitions to later life phases such as `Middle School` and the formation of `Friendship`. Notably, the model also captures the persistent, low-probability risk of a `Health Issue` occurring at any point during this period, demonstrating its ability to model both sequential and parallel events.}
    \label{fig:earlylife}
\end{figure}

\begin{figure}[h!]
    \centering
    \includegraphics[width=1.0\linewidth]{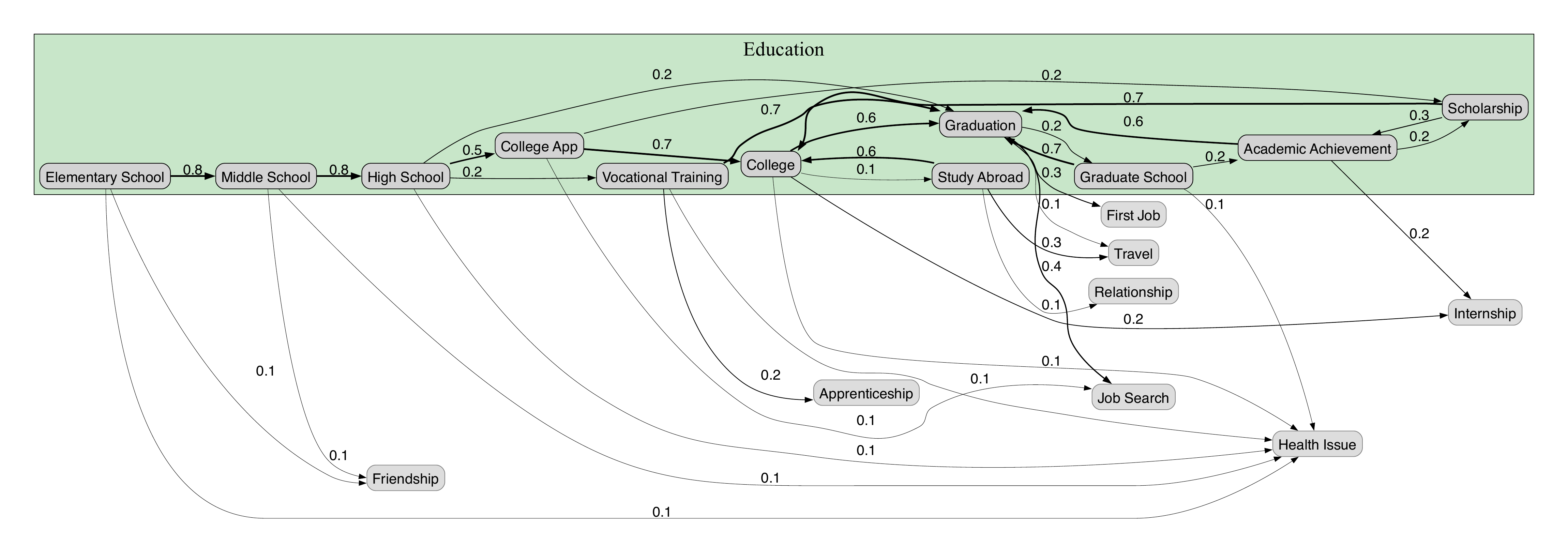}
    \caption{The Education life stage, detailing the probabilistic paths from elementary school through higher education. Nodes within the colored block represent educational milestones, while gray nodes indicate transitions to external stages like Career or Health. The numbers and thickness of the arrows correspond to the transition probabilities, highlighting the primary path towards graduation while showing alternative routes like vocational training or internships.}
    \label{fig:education_path}
\end{figure}

\begin{figure}[h!]
    \centering
    \includegraphics[width=1.0\linewidth]{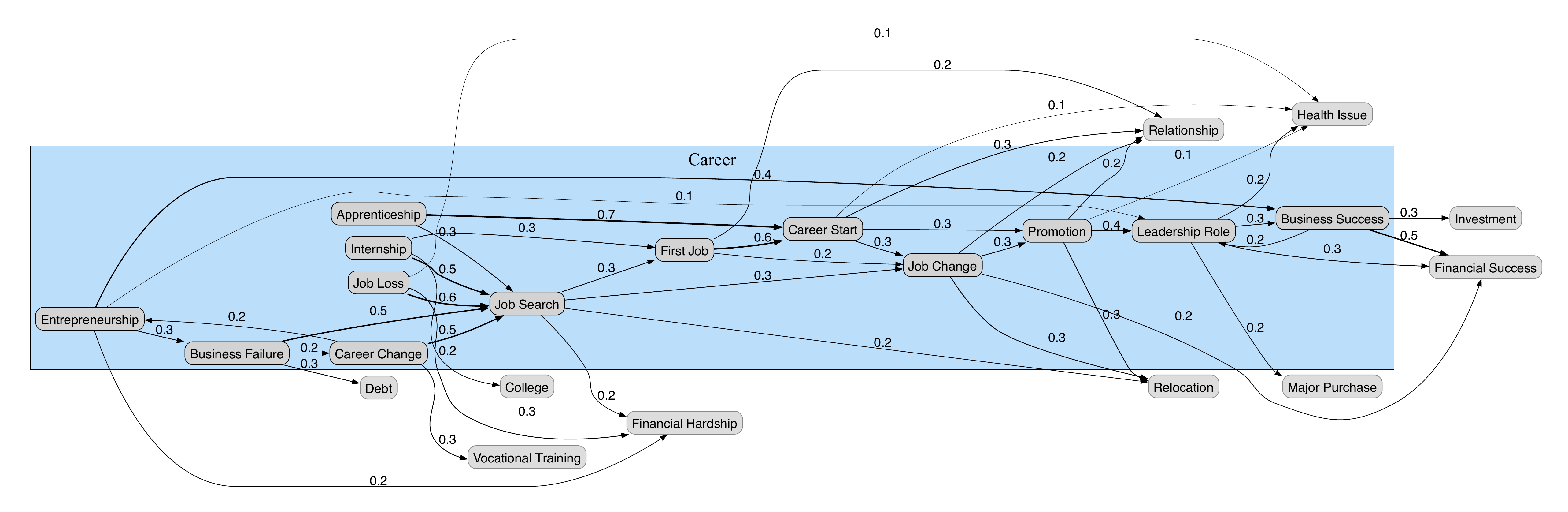}
    \caption{A model of the Career life stage, illustrating the dynamic and non-linear nature of professional life. The diagram shows the typical progression from a first job to a career start, but also includes significant feedback loops such as job loss and job searching. The probabilities on the edges guide the likelihood of events like promotions, career changes, or entrepreneurial ventures, with gray nodes showing strong connections to financial and relationship outcomes.}
    \label{fig:career}
\end{figure}

\begin{figure}[h!]
    \centering
    \includegraphics[width=1.0\linewidth]{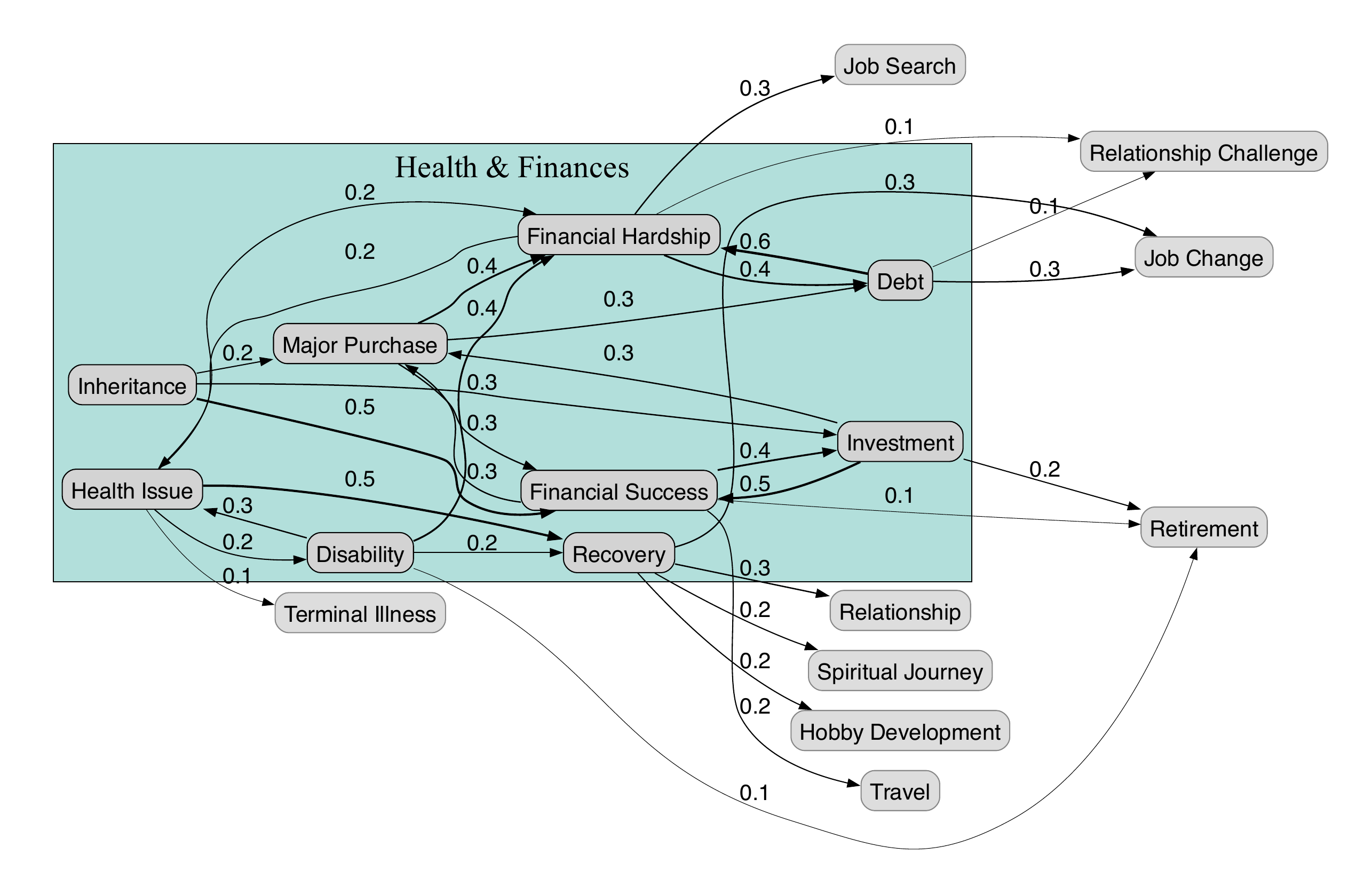}
    \caption{The Health \& Finances block, a central hub that influences all other aspects of life. This model visualizes the critical interplay between health events and financial stability. It contains feedback loops for both positive outcomes, like investment leading to financial success, and negative challenges, such as a health issue leading to financial hardship. The gray nodes show how these events directly impact career, relationships, and later life stages.}
    \label{fig:healthfinances}
\end{figure}

\begin{figure}[h!]
    \centering
    \includegraphics[width=1.0\linewidth]{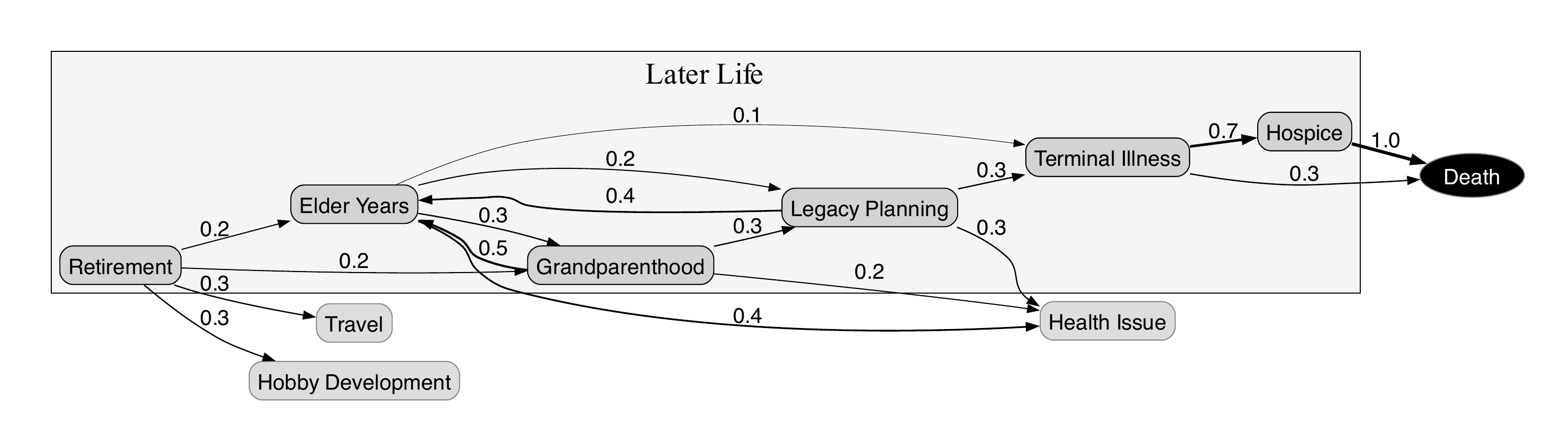}
    \caption{The Later Life stage, modeling the period from retirement to the end of life. The diagram illustrates the shift in focus post-career towards travel, hobbies, and grandparenthood. It also shows the increasing probability of health issues, which create a clear, probabilistic funnel through terminal illness and hospice care to the model's terminal state, Death.}
    \label{fig:laterlife}
\end{figure}

\begin{figure}[h!]
    \centering
    \includegraphics[width=1.0\linewidth]{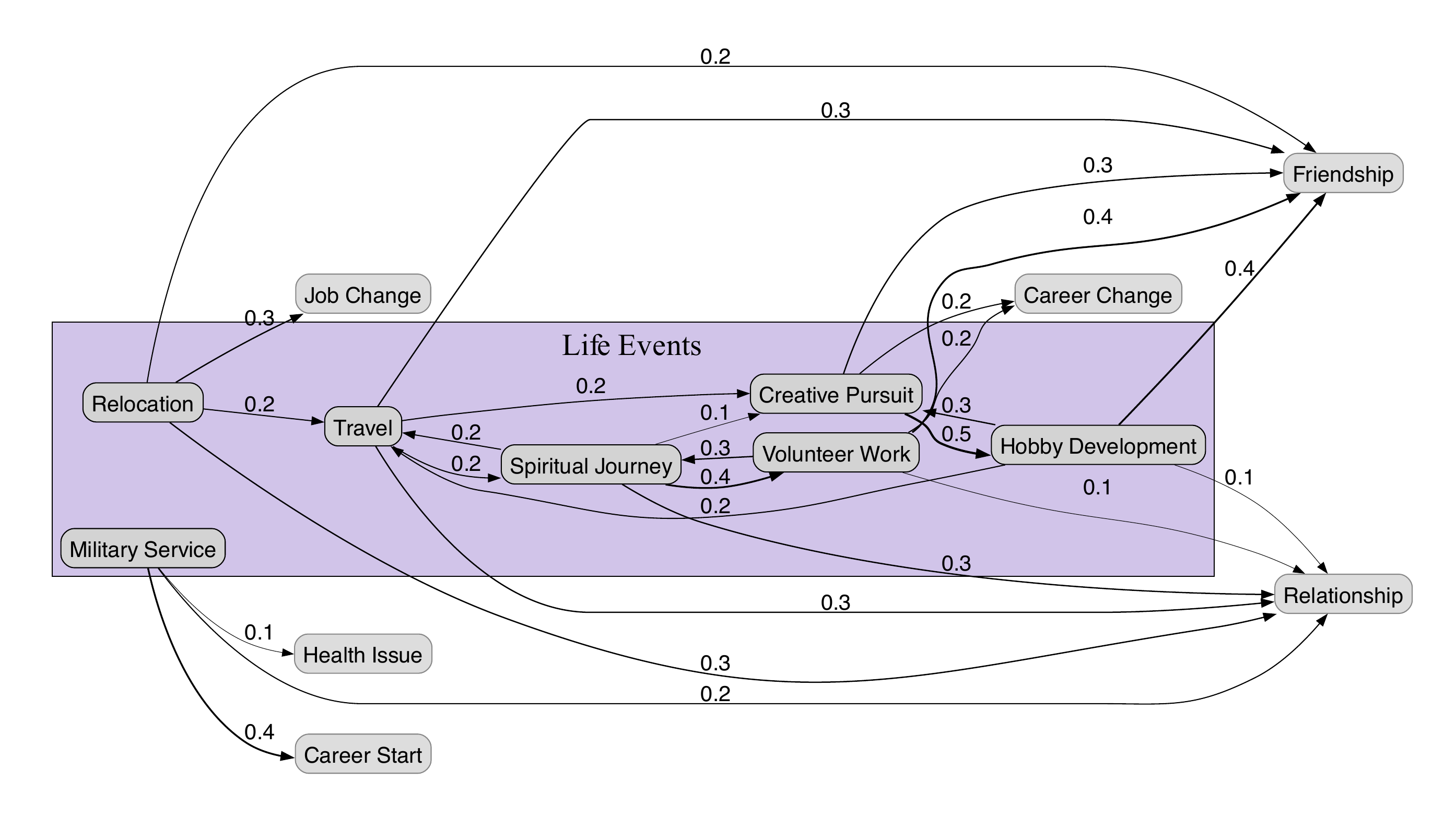}
    \caption{The Life Events block, representing significant personal milestones and pursuits that shape the life path. These events, such as relocation, travel, or military service, often act as catalysts that connect or alter the course of an individual's career and relationships. The graph shows how these pursuits are interconnected and lead to personal development or changes in other life stages.}
    \label{fig:lifeevents}
\end{figure}

\begin{figure}[h!]
    \centering
    \includegraphics[width=1.0\linewidth]{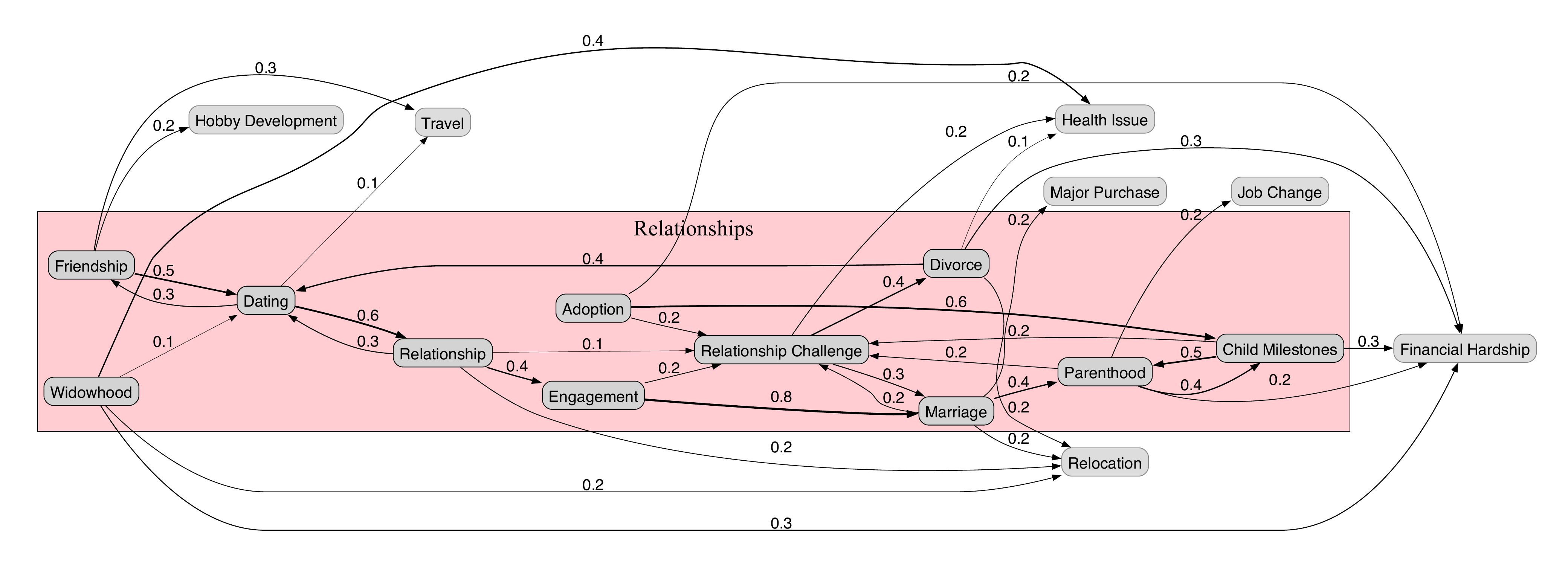}
    \caption{The Relationships life stage, mapping the progression of romantic and family connections. The model follows a common path from dating to marriage and parenthood, with probabilities indicating the likelihood of each step. Crucially, it also incorporates potential setbacks such as relationship challenges or divorce, which create cycles that can return an individual to earlier stages.}
    \label{fig:relationship}
\end{figure}

\clearpage

\subsubsection*{Quantitative Evaluation of Additivity} 
\label{sec:life:add}
We quantitatively validate our linear additive hypothesis by measuring the deviation from ideal additivity as a function of sequence length, as shown in \fig{fig:additivity}. 
In this test, random sequences of increasing length were generated and we compute the cosine similarity between the final hidden state produced by the model's recurrent updates and the ideal state calculated via a direct vector sum of the event embeddings. 
The results from \fig{fig:additivity} confirm a strong additive structure, with the observed gradual decay in similarity for longer sequences being an expected trade-off for the numerical stability provided by the norm clipping regularization discussed in Section 6.

\begin{figure}[h!]
\includegraphics[width=\linewidth]{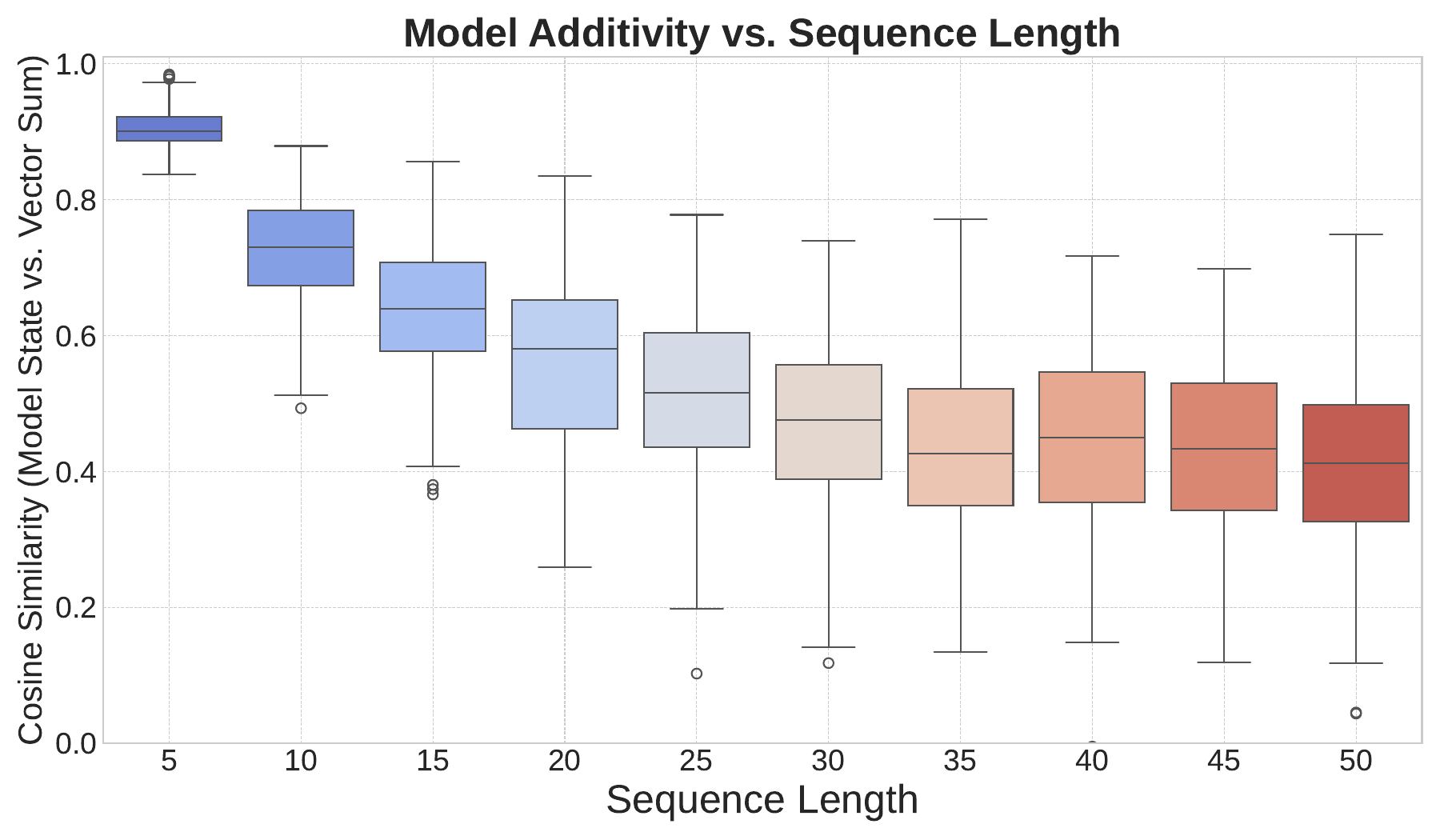}
\caption{\textbf{Quantitative Validation of Compositionality.} Cosine similarity between the final hidden state and the ideal vector sum of its parts remains high (~0.7), even for long sequences. This confirms the model maintains the Linear Additive Hypothesis effectively.}
\label{fig:additivity}
\end{figure}

\subsubsection*{Detailed Justification of Analogical Embedding}
The analogical reasoning results, presented in the \tabref{tab:analogy}, show the ability to capture data-driven relationships that go beyond simple chronological order. 
By analyzing the \textbf{Analogy Query} for each \textbf{ID}, we can interpret the geometric logic behind the Model's Result.

In \textbf{1A}, the query attempts to map the transition from birth to school onto a career context. However, the model interprets the vector `elementary\_school` - `birth` as the literal concept of "childhood development". 
When this vector is illogically added to `first\_job`, the model defaults to the most representative childhood event it knows, `late\_childhood`, although with a moderate cosine similarity of 0.4504. Conversely, in \textbf{1B}, the model correctly interprets the vector `death` - `retirement` as a "final life transition". 
Applying this concept to `graduation`, it identifies `internship` as the corresponding next step, representing the transition from academic life to professional life.

The tests in group 2 explore professional trajectories. In \textbf{2A}, the model associates the vector for career advancement `promotion` - 'career\_start` with `military\_service`. 
This suggests the model has learned that both a promotion and military service are significant, structured steps that follow initial employment, resulting in a relatively high cosine similarity of 0.5048.
Analogy test case \textbf{2B} shows the model clustering concepts of major positive financial events where the vector representing the outcome of entrepreneurship, when applied to an investment, logically yields another major financial event `inheritance` with high similarity score 0.5976.

Finally, the tests in group 3 examine relationship dynamics. 
Row \textbf{3A} shows a good example, where the model correctly identifies `adoption` as a direct parallel to `parenthood` when applying the "formalization" vector from `marriage` - `engagement`. In row \textbf{3B}, the model displays a fascinating insight: when asked to find the outcome of financial hardship by applying the vector for a relationship ending (`divorce` - `relationship\_challenge`), it returns `major\_purchase`. 
This indicates the model has likely learned from the data that a major purchase is a common cause of financial hardship, linking the two events causally rather than sequentially.
    
\end{document}